\documentclass{article}
\PassOptionsToPackage{numbers, compress}{natbib}

\usepackage[preprint]{neurips_2023}

\usepackage[utf8]{inputenc} % allow utf-8 input
\usepackage[T1]{fontenc}    % use 8-bit T1 fonts
\usepackage[hidelinks]{hyperref}       % hyperlinks
\usepackage{url}            % simple URL typesetting
\usepackage{booktabs}       % professional-quality tables
\usepackage{amsfonts}       % blackboard math symbols
\usepackage{nicefrac}       % compact symbols for 1/2, etc.
\usepackage{microtype}      % microtypography
\usepackage{xcolor}         % colors

\usepackage{amsmath}
\usepackage{amsthm}
\usepackage{amssymb}
\usepackage[inline]{enumitem}
\usepackage{todonotes}
\usepackage{comment}
\usepackage{thm-restate}
\usepackage{tikz}
\usepackage{xargs}
\usepackage{xifthen}
\usepackage{multirow}
\usepackage[linesnumbered,lined,ruled,noend]{algorithm2e}
\usepackage{adjustbox}
\usepackage{bbm}

\newcommand{\norm}[1]{\left\| #1 \right\|}
\newcommand{\br}[1]{\left\lbrace #1 \right\rbrace}
\newcommand{\term}[1]{\left( #1 \right)}

\newcommand{\eps}{\varepsilon}

\newtheorem{theorem}{Theorem}
 
 \newtheorem{lemma}[theorem]{Lemma}
 
\newtheorem{definition}[theorem]{Definition}

\newtheorem{Definition}[theorem]{Definition}
 \newtheorem{remark}[theorem]{Remark}
 \newtheorem{claim}[theorem]{Claim}

\newcommand{\abs}[1]{\left| #1 \right|}
\newcommand{\REAL}{\mathbb{R}}

\newcommand{\RANGES}{\mathrm{ranges}}
\newcommand{\smartInit}{\textsc{Smart-Initialization}}
\newcommand{\smartPick}{\textsc{Smart-Pick}}

\newcommand{\rbf}[2][]{\ifthenelse{\isempty{#1}}{e^{-\norm{#2 - x}^2_2}}{e^{-\norm{#2 - #1}^2_2}}}

\newcommand{\say}[1]{``#1"}
\newcommand{\fourier}{\textsc{Random-Fourier-Features}}

\usepackage{thmtools} 
\usepackage{thm-restate}

\usepackage{subcaption}

\makeatletter
\newcommand{\printfnsymbol}[1]{%
  \textsuperscript{\@fnsymbol{#1}}%
}
\makeatother
\title{Dataset Distillation Meets Provable Subset Selection}

\author{%
  {\href{https://scholar.google.com/citations?user=721xaz0AAAAJ&hl=en}{\color{blue}Murad Tukan \thanks{Equal Contribution}}} \\
  \color{magenta}DataHeroes\\
  \color{magenta}Israel \\
  \texttt{murad@dataheroes.ai} \\
  \And
  {\href{https://scholar.google.com/citations?user=6r72e-MAAAAJ&hl=en}{\color{blue}Alaa Maalouf \printfnsymbol{1}}}\\
  \color{magenta}CSAIL,  MIT\\
  \color{magenta}Cambridge, Massachusetts\\
  \texttt{alaam@mit.edu}
  \And
  {\href{https://scholar.google.com/citations?user=nZEtlZoAAAAJ&hl=en}{\color{blue}Margarita Osadchy}}\\
  {\color{magenta}The University of Haifa}\\
  {\color{magenta} Haifa, Israel}\\
  \texttt{rita@cs.haifa.ac.il}
}

\begin{document}

\maketitle

\begin{abstract}
Deep learning has grown tremendously over recent years, yielding state-of-the-art results in various fields. However, training such models requires huge amounts of data, increasing the computational time and cost. 
To address this, dataset distillation was proposed to compress a large training dataset into a smaller synthetic one that retains its performance -- this is usually done by (1) uniformly initializing a synthetic set and (2) iteratively updating/learning this set according to a predefined loss by uniformly sampling instances from the full data.  
In this paper, we improve both phases of dataset distillation: (1) we present a provable, sampling-based approach for initializing the distilled set by identifying important and removing redundant points in the data, and (2) we further merge the idea of data subset selection with dataset distillation, by training the distilled set on ``important'' sampled points during the training procedure instead of randomly sampling the next batch.  
To do so, we define the notion of importance based on the relative contribution of instances with respect to two different loss functions, i.e., one for the initialization phase  (a kernel fitting function for kernel ridge regression and $K$-means based loss function for any other distillation method), and the relative cross-entropy loss (or any other predefined loss) function for the training phase.
Finally, we provide experimental results showing how our method can latch on to existing dataset distillation techniques and improve their performance.  
\end{abstract}

\begin{figure*}[t]
    \centering
    \includegraphics[width=\textwidth]{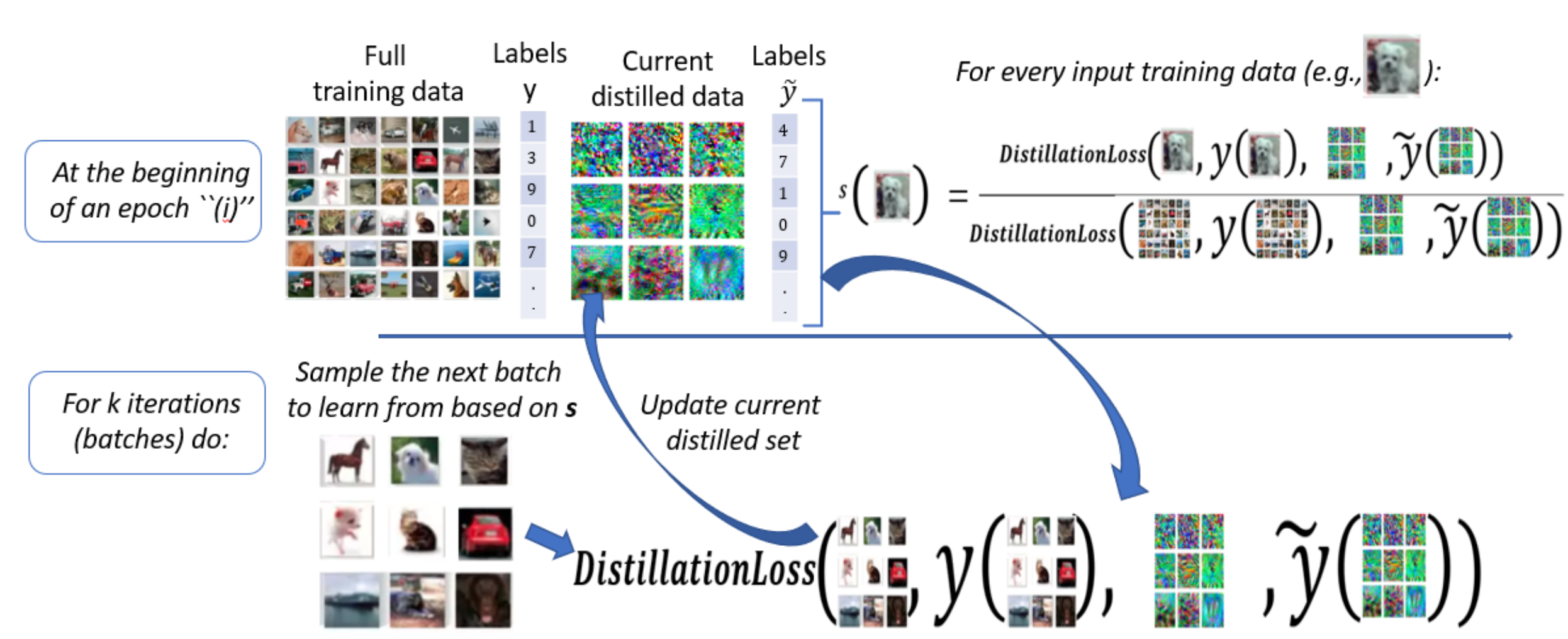}
    \caption{\textbf{Smart Learning. } At the beginning of a specific epoch $i$, we first compute the importance of every input data from the training data. The importance of an input data item aims to measure how well this specific example has been learned/represented using the current distilled data. We then sample the next batches according to this measure. $y$ corresponds to the true label and $\tilde{y}$ denotes the prediction of the model, trained on the distilled dataset; see Section~\ref{sec:smartpick} for more details. }
    \label{fig:smartpick}
\end{figure*}

\section{Introduction}
\label{sec:intro}
%\todo[inline]{Need to elaborate more}

\noindent Deep learning has advanced at an extraordinary rate in the previous decade, becoming the first option in many fields. This advancement is attributed to constantly increasing computational resources that stimulate cutting-edge algorithms to cope with large amounts of data. However, dealing with the infinite expansion of data under a constraint of limited computing power has been increasingly difficult~\cite{nguyen2021dataset}, especially during the training phase, where training a single deep learning model could last for weeks or even months. To that end, subset selection~\cite{mirzasoleiman2020coresets,killamsetty2021grad} and dataset distillation~\cite{loo2022efficient,nguyen2021dataset} methods have been proposed  to speed up the training process.

\textbf{Subset selection.} The idea is to select a representative subset from the entire data~\cite{killamsetty2021glister,tukan2023provable}, before each training epoch, followed by training the model on that subset. However, these methods fall short when such a subset is computed once for the entire training procedure.

\textbf{Dataset distillation.} More recent approaches focused on creating a small synthetic dataset (not necessarily a subset) such that training the model on this small data will approximate the training process on the entire big data. This technique is known as dataset distillation~\cite{loo2022efficient,wang2018dataset} and it has been gaining great interest in recent years. Notably, these synthetic datasets offer the advantage of employing continuous gradient-based optimization procedures rather than geometrical and combinatorial approaches and are not restricted to the collection of inputs and labels provided by the dataset, resulting in increased flexibility and better performance in practice.

Usually, dataset distillation methods start by \textbf{initializing at random} a synthetic dataset or by sampling \textbf{uniformly} a subset from the input data. Then, \textbf{learning the small synthetic} data is done iteratively by sampling batches \textbf{uniformly at random} from the entire data and minimizing a distillation objective, which depends on the algorithm, for example, matching gradients with respect to the data and the loss function used by the model~\cite{killamsetty2021grad}, using the neural tangent kernel of the model~\cite{nguyen2020dataset}, using a teacher/student model mechanism~\cite{cazenavette2022dataset}, or using neural network Gaussian process techniques~\cite{loo2022efficient}. 
Hence, dataset distillation, irrespective of the method used, includes the initialization step and the learning step. 

Then, the learning step itself can be viewed as a process of two functions: The first function \textbf{squeezes} the (current) information to be learned from the whole data set (in previous approaches, this corresponds to the random sampling at each iteration). The second function \textbf{injects}  the extracted information into the synthetic samples (this is the learning process). Thus, most data distillation methods rely heavily on uniform sampling in both the initialization step and in learning, where squeezing is performed by sampling random batches of real data. To this end, we raise the following questions regarding the effectiveness of the uniform sampling and provide the intuition behind them: 

\hypertarget{question1}{(1)} \textbf{Is uniform sampling sufficient for dataset distillation initialization?} 
In general, initialization has a crucial role in determining convergence of the training process of deep learning model~\cite{goodfellow2016deep}. We argue that the initialization of distillation is as important as model initializing for training, and has a similar effect on the convergence and final solution. While each instance is equally important from the point of view of the uniform sampler, the reality of such a matter is far from this, namely, some subsets are more informative than others. Using subset selection with theoretical guarantees for initializing the distilled set offers both practical and theoretical benefits: a) the learning starts from a better solution that is closer to a robust local minimum; b) the training process can only further improve the guarantees associated with the initial subset.

\hypertarget{question2}{(2)} \textbf{Should we apply uniform sampling in the squeezing part at each iteration?} As the learning process proceeds, more and more information from different data points is captured in the distilled set. Thus, the loss of the learning process with respect to these points gets lower and lower, and consequently, the gradients with respect to these points (specifically toward the end of the learning process) get lower as well. Eventually, uniform sampling creates a batch full of ``small norm'' gradients with very few samples affecting the gradient mean. This causes lower diversity of high gradient samples. To resolve this issue we suggest replacing uniform samples with ``important'' examples -- those that the distillation process has not yet encapsulated. In summary, we can obtain better distillation results by sampling a meaningful subset of points at each iteration and learning the synthetic data with respect to them.

\section{Key ideas}
\begin{figure*}[t]
    \centering
    \includegraphics[width=\textwidth]{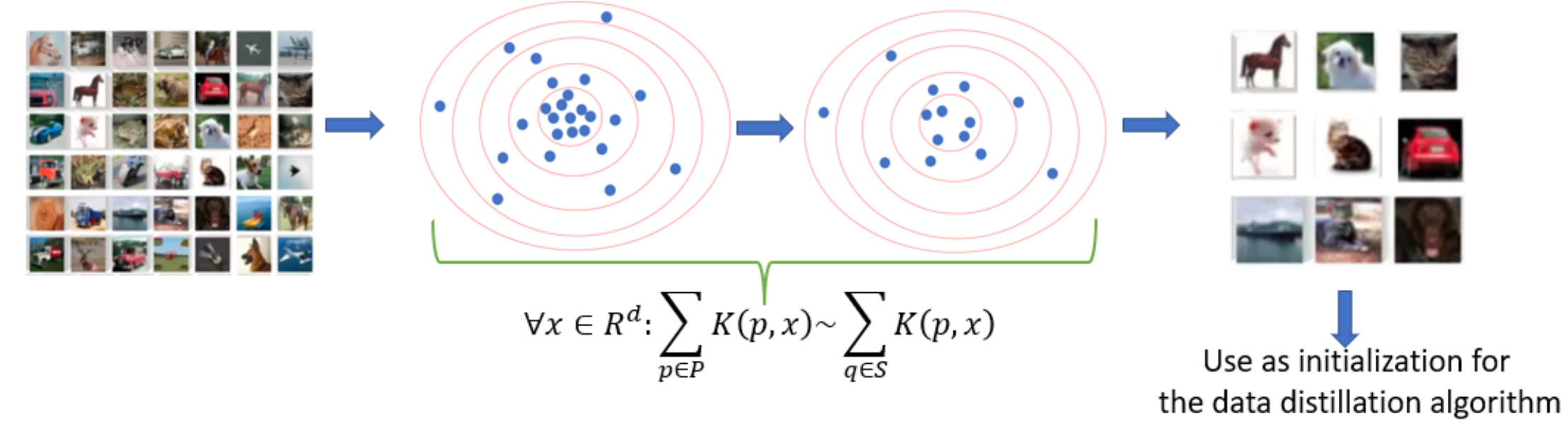}
    \caption{\textbf{Initialization:} Given an input dataset, e.g., CIFAR10, treat the data as a set $P$ of $n$ points in $\REAL^d$, then, apply a coreset construction algorithm (Algorithm~\ref{alg:initAlg}) on $P$ to obtain the coreset $S$ satisfying Theorem~\ref{thm:init}. Use corresponding input data (of $S$) as an initialization for the distillation algorithm. Note that the kernel $K$ we wish to estimate here is the corresponding kernel of the desired neural network we wish to distill data with.}
    \label{fig:galaxy}
\end{figure*}
Following the motivation described in the previous section, we introduce the two main ideas.  
\subsection{Provable initialization} In this work, we distinguish between two main panels of distillation techniques: neural tangent kernel (NTK) based methods (e.g., \cite{nguyen2020dataset}), and neural network-related loss minimization techniques (e.g.,~\cite{DSA}). For each of the above categories of distillation techniques, we provide a robust initializer with provable theoretical guarantees.

\textbf{Initialization for NTK-based distillation methods.} Recent work~\cite{lee2017deep,jacot2018neural} has shown that the learning procedure of neural networks replicates the behavior of Gaussian kernel as the width of the neural network tends to infinity. In other words, one can regard the neural network as a Gaussian process. This eases the inference and analysis of the learning process and it can be used for dataset distillation~\cite{nguyen2020dataset}. For formal recitation, the following is given.

\begin{claim}[ANNs represent Gaussian kernels~\cite{jacot2018neural}]
\label{clm:NTKGaussian}
In the infinite-width limit, artificial neural networks have a Gaussian distribution described by a kernel.
\end{claim}

In other words, every neural network tends to behave like a Gaussian process from the perspective of learning -- the wider the width of the neural network, the more similar the network is to a Gaussian process. This means that the corresponding kernel (referred to in the literature by neural tangent kernel~\cite{jacot2018neural}) of the neural network (describes the learning process) describes a Gaussian distribution; for full details on how to compute this kernel, we refer the reader to~\cite{jacot2018neural}.   
With this in mind, Question~\hyperlink{question1} {(1)} for NTK-based methods can be restated as follows, \say{\emph{Can we find a subset of the data such that fitting the infinite-width neural network corresponding kernel on this subset will approximate fitting the same kernel on the entire data?}}. Note that such a subset is attainable (as we will see later) for the case of infinite-width models, where the associated kernel is Gaussian in nature~\cite{jacot2018neural}. In addition, using such a set in dataset distillation is beneficial, as the subset encapsulates the capabilities and caveats of the learning process of the model on the entire data.

\textbf{Initialization for neural network-related loss minimization distillation technique (\emph{NNLMDT}).} Recently~\cite{cui2022dc} questioned the effect of initialization of the dataset distillation techniques~\cite{zhao2020dataset,DSA,zhao2023dataset}, and observed that using $K$-center initialization~\cite{sener2017active} of the synthetic set improves the accuracy of the learned model trained on the distilled data. With this in mind, Question~\hyperlink{question1} {(1)} in the case of \emph{NNLMDTs} can be restated as follows, \say{\emph{Can we improve upon $K$-center initialization in the case of NNLMDT?}}.

\subsection{Data distillation meets subset selection} For small sizes of distilled datasets, the performance of data distillation techniques is far superior to that of subset selection techniques~\cite{yu2023dataset}. On the other hand, subset selection techniques identify the relevant subset of the training data by either finding ``hard'' examples with respect to the model's success in classification~\cite{paul2021deep} or by selecting a subset such that the model's gradient with respect to the entire data is approximated by the model's gradient with respect to the subset~\cite{killamsetty2021grad,tukan2023provable}, etc. Our second idea is to take the best from both worlds by merging both methodologies. 

We propose to replace the uniform sampling used in the distillation methods 
(e.g.,~\cite{loo2022efficient}) with randomized greedy subset selection to form batches in model training. This modification aims to draw model's attention to the harder instances (associated with the higher loss) to learn the distilled set. With this goal in mind, Question~\hyperlink{question2}{(2)} becomes 
\say{\textbf{How to determine which instances of the data should be selected for the  next update of the distilled data that yields more effective learning?}}

\subsection{Our contribution}
In this paper, we answer the questions raised in the previous section by suggesting the following:
\begin{enumerate}[label=(\roman*)]
    \item \textbf{NTK Provable-Init:} A provable data subset selection for better initialization of the NTK-based learning dataset distillation algorithms. 
    Our algorithm guarantees that fitting any Gaussian kernel on the initially selected subset (to be learned) approximates the fitting cost of using the same kernel on the entire data; see Theorem~\ref{thm:init}, and Algorithm~\ref{alg:initAlg}. 
    This implies that the learning procedure of models with equivalence to Gaussian kernels will behave the same on our subset (to some degree $\eps$) and on the entire data. To do so, we provide the first multiplicative factor approximation coreset for kernel ridge regression. For full details see Theorem~\ref{thm:init} and Remark~\ref{remark:init}, and for an illustrative depiction of this result see Figure~\ref{fig:galaxy}.
    While our application hinges upon approximating Gaussian kernels, our theoretical guarantees hold for a family of kernel functions, namely, positive-definite shift-invariant kernel functions. \label{contrib_1}

    \item \textbf{\emph{NNLMDT} Provable-Init:} We propose the usage of $K$-means coresets (specifically lightweight) as an initialization technique for non-NTK-based distillation techniques. We believe that this initialization technique better encapsulates the underlying structure of the data than the $K$-center problem; see Section~\ref{sec:kmeans_init_coreset}, and Figure~\ref{fig:smartInitResults} at Section~\ref{sec:exp}.\label{contrib_1.5}
    
    \item \textbf{Better sampling during training:} A provable data subset selection policy to replace the uniform sampling which is used for updating the distilled dataset at every iteration of the distillation process. The algorithm enjoys some provable guarantees with respect to the classification capabilities associated with the distilled dataset. Using this policy, we direct the focus of the distillation process toward instances that are not encapsulated in the scope of the distilled data; see Algorithm~\ref{alg:smartPick}, Theorem~\ref{thm:robin_heart} and Figure~\ref{fig:smartpick}.\label{contrib_2}

    \item \textbf{Distil-Boost:} Combining~\ref{contrib_1},~\ref{contrib_1.5}, and~\ref{contrib_2}, yields our proposed boosting mechanism for any sampling-based distillation techniques.
    
    \item Extensive experimental study and open source code of distillation algorithms on various datasets to show how our method improves the quality of the distilled dataset.
\end{enumerate}

\section{Method}
\label{sec:preliminaries}

We now present our notations and definitions and restate theorems from related work to fit our context.

\textbf{Notations. }For an integer $n>2$, we use $[n]$ to denote the set $\{1,2,\ldots,n\}$. A weighted set of points is a pair $(P,w)$, where $P \subset\REAL^d$ and $w : P \to [0, \infty)$ is a weight function. We use $\mathbf{1}_n \in \REAL^n$ to denote a vector of $n$ entries each equal to $1$. For any set $A \subseteq \REAL^d$, $\abs{A}$ denotes the cardinality of $A$.

As stated in Contributions~\ref{contrib_1} and~\ref{contrib_1.5}, we provide an initialization policy towards better learning the distilled set. 
To suggest such initializations, we turn to coresets. A coreset is a small weighted subset $(S,v)$ of the full dataset $P\subset \REAL^d$ that provably approximates the given loss function $f$ for every query in a given set of queries $X$. Coresets~\cite{bachem2018scalable,pmlr-v84-bachem18a,buadoiu2008optimal,maalouf2022fast,balcan2013distributed,maalouf2023autocoreset,pmlr-v97-braverman19a,curtain2019coresets} have been suggested for various machine learning problems as they were found to be advantageous in many aspects; e.g., for clustering~\cite{jubran2020sets,maalouf2021coresets,braverman2021coresets,tukan2022new}, for regression and classification problems~\cite{tukan2020coresets,NEURIPS2021_90fd4f88,maalouf2019fast}, sine wave fitting and $k$-segmentation problems~\cite{NIPS2014_bca82e41, maalouf2022coresets}, coresets for pruning neural networks~\cite{tukan2022pruning, mussay2021data,maalouf2022unified}, etc; see surveys~\cite{jubran2019introduction,feldman2020introduction}. In addition, coresets were utilized for marine application~\cite{DBLP:journals/corr/abs-2301-04216} and in the field of Robotics\cite{9981428}, specifically, for boosting the quality of sampling based path planners, e.g., RRT$^\ast$~\cite{5980479}. For formality, the following definition is given.
\begin{definition}[$\eps$-coreset]
\label{def:epsCore}
Let $(P,w)$ be a weighted set of $n$ points in $\REAL^d$, $X$ be a (probably infinite) set of queries, $\eps \in (0,1)$, and let $f : P \times X \to [0,\infty)$ be a loss function. An $\eps$-coreset for $(P,w,X,f)$ is a pair $(S,v)$ where $S \subseteq P$, $v:S\to [0,\infty)$ is a weight function, such that for every $x \in X$, $\abs{1 -  \frac{\sum_{q \in S} v(q)f(q,x)}{\sum_{p \in P} w(p) f(p,x)}} \leq \eps.$
\end{definition}

\subsection{Smart initialization via Gaussian kernel coresets for NTK-based distillation methods} 
Claim~\ref{clm:NTKGaussian} states that as neural networks tend to be wider, their NTK (the kernel which approximates their learning behavior) tends to be Gaussian. If we were to find a subset of the data that approximates a Gaussian kernel fitting problem, this subset would help improve the distillation process.
To this end, we define the kernel function $k: \REAL^d \times \REAL^d$, such that for every $x,y \in \REAL^d$, $k(x,y) = e^{-0.5\norm{x - y}^2},$ where the fitting function that our coreset aims to approximate, is the kernel density loss, i.e, for every  $x \in \REAL^d$: $\sum\limits_{p \in P} \frac{1}{n}k(p,x).$ 
The motivation behind this loss function is simple as it aims to encapsulate the probability distribution function of the data. Thus, finding a coreset for the Gaussian kernel density estimation gives a coreset that enables the preservation (to some extent) of the properties associated with the neural network tangent kernel concerning the training data. 
For the initialization of NTK-based methods, we aim at finding the best subset that ensures that the ratio between the kernel fitting cost of the entire data and that of the chosen subset is close to $1$. Hence, we aim to compute an $\eps$-coreset $(S,v)$ for the tuple $(P,w,\REAL^d,k)$. In the context of distillation initialization, $P$ here is the input training data (e.g., CIFAR10~\cite{krizhevsky2009learning}), and for every input $p\in P$, $w(p)=1/n$.

\textbf{Sensitivity sampling.} 
To compute an $\eps$-coreset, we utilize the sensitivity sampling framework~\cite{braverman2016new}. 
In short, the sensitivity of a point $p\in P$ is a number$ s(p) \geq \sup_{x \in X} \frac{w(p)f(p,x)}{\sum_{q \in P} w(q)f(q,x)}$, 
(where the sup is over every $x \in X$ such that the denominator is non-zero) that reflects the ``importance'' of this point with respect to the whole input set $P$ and the loss function at hand. Once we bound the sensitivities for every $p\in P$, we can sample points from $P$ according to their corresponding sensitivity bounds and re-weight the sampled points to obtain an $\eps$-coreset as in Definition~\ref{def:epsCore}.  
The size of the sample (coreset) is proportional to the sum of these bounds -- the tighter (smaller) these bounds, the smaller the coreset size. Formally speaking, Theorem~\ref{thm:coresetbraverman} is a simplified restatement of the generic Theorem 5.5 in~\cite{braverman2016new} which explains how to compute an $\varepsilon$-coreset for the input data with respect to a general loss function (query space) $f : P \times \REAL^d \to [0, \infty)$ (in our case $f:=k$).

\textbf{Sensitivity bounding.} To generate our coresets, we first need to upper-bound the sensitivities.  Recall that a kernel corresponds to an inner product between mappings of a pair of points in some feature space~\cite{theodoridis2006pattern}. Assume for now, that such mapping $z$ is explicitly given, i.e., $z: \REAL^d \to \REAL^c$ such that for any pair $x,y \in \REAL^d$, $k(x,y) = z(x)^Tz(y)$ (this assumption will be handled/removed later). Thus, the loss function we wish to approximate becomes $\sum_{p\in P} z(p)^Tz(x)$ for every $x \in \REAL^d$. We note that for the Gaussian kernel (and many other kernels), $z(x)^T z(y) = \abs{z(x)^T z(y)} \geq 0$ for any $x,y \in \REAL^d$. Hence to bound the sensitivity, we define  $\tilde{P} = \br{z(p)| p\in P}$ and bound $s(p)=\sup_{x\in \REAL^d}  \frac{z(p)^Tz(x)}{\sum_{q\in P}  z(q)^Tz(x)}$ using Lemma~\ref{lem:sens_bound} in Section~\ref{sec:proof_of_init_theorem} in the supplementary material.

\textbf{On constructing $z$.} The problem with this approach is that the feature space ($\REAL^c$) is rather intangible or hard to map to, while, the sensitivity bounding tool in Lemma~\ref{lem:sens_bound} requires the representation of these points in the feature space. 
To resolve this problem, we use a randomized approximated feature map -- the dot product in this approximated feature map aims to approximate the Gaussian kernel function. Specifically, we use Theorem~\ref{thm:rahimi} (see Section~\ref{sec:fourier} in the supplementary material) which states that there exists a feature space of some dimension $D \ll c$ such that the dot product of mapped features of two points in this space approximates the original Gaussian kernel which the model with the infinite width's kernel mimics. Hence, we first aim at computing a coreset for the dot product loss function in such spaces. This coreset will be used to compute the initial starting subset for any NTK-based sampling-based distillation techniques.

In summary, Algorithm~\ref{alg:initAlg} and Theorem~\ref{thm:init} describe the pseudo-code and theoretical guarantees associated with our kernel fitting problem for the NTK-based methods initialization.

\subsection{Smart initialization via $K$-means coresets for NNLMDT distillation methods}
\label{sec:kmeans_init_coreset}

First observe that the $K$-center problem which inspired the use of $K$-center initializer~\cite{sener2017active} involves the optimization problem of $\min_{\substack{C \subset \REAL^d \\ \abs{C} = K}} \max_{p \in P} \min_{c \in C} \norm{p - c}_2^2.$
While, such a problem is an NP-hard problem, the $K$-center initialization is deterministic and easy to compute and is a greedy algorithm that admits a multiplicative factor of $2$ to the optimal $K$-center problem~\cite{sener2017active}. 
We believe that while such an optimization problem encapsulates the internal structure of the data to some degree, the $K$-means problem can be exploited to gain higher accuracy than that of the $K$-center problem, and we believe that this boils down to the optimization problem that $K$-means involves, which is, $\min_{\substack{C \subset \REAL^d \\ \abs{C} = K}}\sum\limits_{p \in P} \min_{c \in C} \norm{p - c}_2^2$. Thus, we propose using (lightweight) $K$-means coreset~\cite{bachem2018scalable}, the lightweight coreset, that aims to approximate the $K$-means problem that involves the input data while admitting $(1+\eps)$-multiplicative approximation to any given $K$-points determining the clustering of the input data. Formally, for a given set $P\subset \REAL^d$, a $K$-means coreset is a pair $(S,u)$, where $S\subset P$ and $u:S\to \REAL$ such that for a any set $C$ of $K$ points in $\REAL^d$ 
$\abs{1 - \frac{\sum_{p \in S} \min_{c\in C} u(p)\norm{p - c}_2^2}{\sum_{p \in P} \min_{c\in C} \norm{p - c}_2^2} } \leq \eps.$ Note that the $K$-means problem differs from the $K$-center problem by the fact that instead of using the max operator over the distance between each point and its closest center (native to the $K$-center problem), the problem revolves around averaging the distance of points to their closest center.
We stress that by doing so, we better encapsulate the underlying structure of the data than the $K$-center problem.  Next, we provide our first main result: Algorithm~\ref{alg:initAlg} which includes initialization methods for both NTK-based distillation methods and \emph{NNLMDT}s.  Theorem~\ref{thm:init} provides the theoretical guarantees for Algorithm~\ref{alg:initAlg}.

\subsection{Provable initializer: Overview of Algorithm~\ref{alg:initAlg}}

Given a set $P$ of $n$ points in $\REAL^d$, a weight function $w : P \to [0,\infty)$, a sample size $m > 0$, and a kernel function $k$, 
for a sufficiently large sample size $m$, Algorithm~\ref{alg:initAlg} outputs a pair $(S, v)$ that is an $\eps$-coreset with respect to either the kernel fitting problem involving $k$, or the $K$-means clustering problem involving $P$ for any $K \geq 1$. Note that for \emph{NNLMDT}s distillation techniques, the function $k$ is set to an empty function. \textbf{In the case of NTK-based:} The heart of our algorithm lies in approximating the feature map function associated with the kernel function $k$ using~\cite{rahimi2007random}. We aim to approximate the structural properties associated with such space using coresets for the $\ell_1$-regression problem, specifically that of~\cite{tukan2020coresets}. This is done by ensuring that the dot product of the feature map is non-negative as a result of carefully choosing the approximation used by Theorem~\ref{thm:rahimi}; we refer the reader to the proof of Theorem~\ref{thm:init}. The conversion from kernel fitting to dot product of points in the feature space requires manipulation of the input data as presented at Line~\ref{algLine:manipulate}. We then compute the $f$-SVD of the new input data with respect to the $\ell_1$-regression problem followed by bounding the sensitivity of such points (Lines~\ref{algLine:l1svd}--\ref{algLine:sensBound}). \textbf{As for \emph{NNLMDT}:} we first compute the weighted average $\mu$ of $P$ as in Line~\ref{algLine:weightedAverage}, followed by setting the sensitivity to be the average between (i) the relative squared distance from $\mu$ of any point $p \in P$ with respect to the summation of squared distances from the points of $P$ to $\mu$ and (ii) the relative weight of a point with respect to the total weight of the points of $P$. We adopted this sensitivity bound from~\cite{bachem2018scalable} in order to be able to construct lightweight coresets which mainly aim towards the $K$-means clustering problem.  \textbf{Finally, } we have all the required ingredients to use Theorem~\ref{thm:coresetbraverman} (see Section~\ref{sec:coreset_tools} at the supplementary material) in order to obtain an $\eps$-coreset, i.e., we sample i.i.d $m$ points from $P$ based on their sensitivity bounds (see Line~\ref{algLine:donesamle}), followed by assigning a new weight for every sampled point at Line~\ref{algLine:weight}.

\begin{algorithm}[htb!]
\caption{$\smartInit(P, w, m, k)$\label{alg:initAlg}}
{\begin{tabbing}
\textbf{Input:} \quad\= A set $P \subseteq \REAL^{d}$ of $n$ points, a weight function $w : P \to [0, \infty)$, a sample size $m \geq 1$, 
\\\>and a kernel function $k$.\\
\textbf{Output:} \>A pair $(S,v)$ satisfying Theorem~\ref{thm:init}.
\end{tabbing}}

\If{$\mathbbm{1}\term{\texttt{Is-NTK-based-?}}$}{
Construct $z$ via Theorem~\ref{thm:rahimi} with respect to $k$ \label{algLine:Rahimi} \tcp{See Algorithm~\ref{alg:fourier} at the Appendix} 

Set $P^\prime = \br{z(p) \middle| p \in P}$ \label{algLine:manipulate}\\

Set $(U,\Sigma,V)$ to be the $\ell_1$-SVD of $\term{P^\prime,w,\abs{\cdot}}$\label{algLine:l1svd} \tcp{See Lemma~\ref{lem:sens_bound} at the Appendix}
}
\Else{
Set $\mu \in \REAL^d$ to be $\sum_{p \in P} \frac{w(p)}{\sum\limits_{q \in P} w(q)}  p$ \label{algLine:weightedAverage}\\
}

\For{every $p \in P$ \label{algLine:4}}
{
Set $s(p):= \begin{cases}
 w(p)\norm{U(z(p))}_1 & \text{NTK-based distillation techniques} \\
 \frac{w(p)}{2\sum\limits_{q \in P} w(q)} + \frac{w(p)\norm{p - \mu}_2^2}{\sum_{q \in P} 2w(q)\norm{q - \mu}_2^2} & \text{Otherwise} 
\end{cases}$ \label{algLine:sensBound}\\ 
}
Set $t := \sum_{p \in P} s(p)$ \label{algLine:6} \\

Pick an i.i.d sample $S$ of  
$m$ points from $P$, where each $p \in P$ is sampled with probability $\frac{s(p)}{t}$.  \label{algLine:donesamle}\\

Set $v: \REAL^d \to [0,\infty]$ to be a weight function such that for every $q \in S$, $v(q)=\frac{tw(q)}{s(q)\cdot m}$. \label{algLine:weight} \\
\Return $(S,v)$ \label{algLine:coresetDone}
\end{algorithm}

We now present the theoretical guarantees of Algorithm~\ref{alg:initAlg}.

\begin{restatable}[Guarantees of our smart initialization]{theorem}{init}
\label{thm:init}
Let $P \subseteq \REAL^d$ be a set of $n$ points, $K>0$ be an integer, $R$ be the diameter of $P$\footnote{For \emph{NNLMDT}s, $R$ can be as large as possible since it will not be used, whereas for NTK-based, it will be present in $D$.}, $X$ be a set of queries of diameter $3R$ containing $P$, i.e., $P \subseteq X$. Let $k : P \times X \to (0, \infty)$ be a Gaussian kernel such that $\min_{p \in P, x \in X} k(x,p) = \alpha\term{R}$, where $\alpha\term{R} > 0$, $\delta \in (0,1)$, and let $\eps \in (0, \alpha\term{R}]$. 
Finally, let $D$ be defined as in Theorem~\ref{thm:rahimi} with plugging $X := X$ and $k := k$ and $\eps := \eps$ into Theorem~\ref{thm:rahimi}. 
Let $m \in O\term{\frac{Dd}{\eps^2}\term{d\log{D} + \log{\frac{1}{\delta}}}}$ for NTK based distillation techniques, and $m \in O\term{dK\term{\log(K) + \log{\term{\frac{1}{\delta}}}}}$ for the \emph{NNLMDT}s. 
Let $(S,v)$ be the output of a call to $\smartInit(P, \frac{1}{n}\mathbf{1}_n,m,k)$. Then with probability at least $1-\delta$, for NTK-based distillation techniques, it holds that for every $x \in X$,
$$ 
\abs{\sum\limits_{q \in S} v(q) k(q,x) - \sum\limits_{p \in P} \frac{1}{n}k(p,x)} \leq \eps \term{1 + \sum\limits_{p \in P} \frac{1}{n}k(p,x)},
$$
while for \emph{NNLMDT}s, let $\mu = \sum_{p \in P} p/n$, it holds that for every set $C \subseteq \REAL^d$ of $K$ points
$$\abs{\sum\limits_{p \in P} \min_{c\in C} \norm{p - c}_2^2- \sum\limits_{p \in S} \min_{c\in C} v(p)\norm{p - c}_2^2 } \leq \frac{\eps}{2}  \sum\limits_{p \in P} \min_{c\in C} \norm{p - c}_2^2 + \frac{\eps}{2}  \sum\limits_{p \in P} \norm{p - \mu}_2^2.$$
\end{restatable}

We refer the reader to Section~\ref{sec:extensions} for further extensions of this result.

\subsection{Subset selection policy for better distillation} \label{sec:smartpick}
In the learning stage, batches of real data are sampled from the full dataset to compute the distillation loss.
We propose sampling subsets of the training data that are either not in the scope of the distilled dataset or not fully covered in it. This will enhance the quality of the distilled dataset.
To sample the next batch wisely, we define the notion of \textbf{importance-based sampling strategy}. 
Let $P \subseteq \REAL^{d}$ be the full training data, and $y : P \to \mathcal{L}$ be it corresponding label function, $\tilde{P}$ be the current distilled dataset, and its corresponding labels $\tilde{y}$. Let $\phi$ be the loss function used by the dataset distillation technique for prediction during distillation. For each sample $p$ in the training set $P$, we assign a probability  $\mathbb{P} (p) := \frac{\phi(p, y(p),\tilde{P},\tilde{y})}{\sum_{q \in P} \phi(q, y(q),\tilde{P},\tilde{y})}$ (a.k.a the importance) equal to the ratio between the current value (this value changes during the training process) of the loss function used by the distillation method on this example and the total loss -- sum of losses across the whole  training set. For \emph{KIP}~\cite{nguyen2020dataset}, the loss function is the squared difference loss function on the predicted labels using kernel ridge regression model and their ground truth label, for \emph{RFAD}~\cite{loo2022efficient}, an x-entropy loss function is used to measure the distance between the predicted label and true label, and finally, for \emph{DC}~\cite{zhao2020dataset}, the loss function is the cross-entropy loss; see Figure~\ref{fig:smartpick} for illustration.

\begin{remark}[Weak coreset]
Our importance sampling can also be regarded as a weak coreset where $f$ would be the loss function used by the distillation technique, the query set $X$ in this context contains at each moment of the training procedure of the distillation technique, the current distilled data's properties, e.g., the kernel matrix in \emph{KIP} and \emph{RFAD}, and the network in \emph{DC}. Using the above parts, our importance sampling is parallel to generating a weak coreset that enjoys provable guarantees for a single query. For more details, see Theorem~\ref{thm:robin_heart} and its proof.
\end{remark}

We now present our subset selection for better distillation, referred to as Algorithm~\ref{alg:smartPick}.

\textbf{Overview of Algorithm~\ref{alg:smartPick}.} 
Algorithm~\ref{alg:smartPick} receives as input, a set $P$ of $n$ points in $\REAL^d$, a labeling function $y : P \to \mathcal{L}$ where $\mathcal{L}$ denotes the set of unique labels, a sample size $m > 0$, the current distilled dataset $\tilde{P}$ associated with its labels $\tilde{y}$ at some iteration $i\geq 1$ and a loss function $\phi$ used by some data distillation technique for prediction of training instances during the distillation process. 
First, at Line~\ref{liine:distribution}, we calculate for every point $p\in P$ the ratio between its loss and the sum of losses across all of the training data with respect to the predicted labels of the distillation process.
Then at Line~\ref{liine:sample}, we sample $m$ points according to these computed ratios and assign each sampled point a weight (Line~\ref{liine:weights}). These points form a batch for  updating the distilled set (via the corresponding dataset distillation algorithm).
Algorithm~\ref{alg:smartPick} guarantees a weak coreset with respect to the loss function used by the distillation method at hand. Figure~\ref{fig:smartpick} provides an intuitive illustration for the suggested subset selection sapling policy.

\begin{algorithm}[htb!]
\caption{$\smartPick\term{P, y, m, \phi,\tilde{P},\tilde{y}}$\label{alg:smartPick}}
{\begin{tabbing}
\textbf{Input:} \quad\= A set $P \subseteq \REAL^{d}$ of $n$ points, a label function $y : P \to \mathcal{L}$, a sample size $m \geq 1$,
\\\>current distilled dataset $\tilde{P}$
and its corresponding labels $\tilde{y}$
and a loss function $\phi$ \\\> used by the dataset distillation technique for prediction during distillation.\\
\textbf{Output:} \>A subset $S\subset P$ and $v : S \to [0,\infty)$.
\end{tabbing}}

\For{every $p$ in $P$}{
    $\mathbb{P} (p) := \frac{\phi(p, y(p),\tilde{P},\tilde{y})}{\sum\limits_{q \in P} \phi(q, y(q),\tilde{P},\tilde{y})}$ \tcp{Setting the sampling probability of $p$} \label{liine:distribution}
}
$S := $ Sample $m$ points from $P$ with via importance sampling involving $\mathbb{P}\term{\cdot}$\label{liine:sample} \\
Let $v : S \to [0,\infty)$ such that for every $p \in S$, $v(p) = \frac{1}{\mathbb{P}\term{p} \cdot m}$ \label{liine:weights} \\

\Return $\term{S,v}$ \label{algLine:smartcoresetDone}
\end{algorithm}

We now show the provable guarantees of Algorithm~\ref{alg:smartPick} and their implications.

\begin{restatable}{theorem}{robinheart}
\label{thm:robin_heart}
Let $P$ be a set of $n$ points, $y$ be a labeling function, $\eps,\delta\in(0,1)$, $\phi$ be a loss function used by the distillation technique, $m \in O(\log(1/\delta)/\eps^2)$ be a sample size, and let $(\tilde{P},\tilde{y})$ be the current distilled set (with its labels) at some iteration $i\geq0$. Let $\term{S, v}$ be the output of a call to $\smartPick\term{P, y, m, \phi,\tilde{P},\tilde{y}}$. 
Then, with probability $1-\delta$,
$ \abs{1 - \frac{\sum_{q \in S} v(q)\phi(q, y(q),\tilde{P},\tilde{y})}{\sum_{q \in P} \phi(q, y(q),\tilde{P},\tilde{y}) }} \leq \eps. $
\end{restatable}

\textbf{The logic behind Algorithm~\ref{alg:smartPick}.} Notice that we could have decided to make $S$ (output of Algorithm~\ref{alg:smartPick}) contain only the hardest points to classify from the point of view of the distilled data. However, we wouldn't have encapsulated the classification capabilities of the model properly, as some of the easier examples can be harder for classification in the next iteration of the data distillation technique. Hence, by taking a coreset, we can (i) better encapsulate the distributional properties of the classification capabilities associated with the distilled data from a probabilistic point of view, and (ii) increase the chances of an item that is not that hard in this iteration being hard in the next iteration to be part of $S$.

\section{Experiments}
\label{sec:exp}

\textbf{System details.} Our algorithms were implemented in Python 3.9~\cite{python} using Numpy~\cite{harris2020array}, Scikit-Learn~\cite{scikitlearn}, and
Pytorch~\cite{NEURIPS2019_9015}. Tests were performed on NVIDIA DGX A100 servers with 8 NVIDIA A100 GPUs each, fast InfiniBand interconnect, and supporting infrastructure.

\textbf{Neural network and Datasets.} Throughout our experiments, we used the ConvNet network~\cite{szegedy2015going} for both training and evaluating synthetic dataset. This network is comprised of three $3\times3$ convolution layers, each followed by $2 \times 2$ average pooling and instance normalization. The hidden embedding size is set to $128$. As for the datasets, we have used MNIST~\cite{deng2012mnist}, Fashion-MNIST~\cite{xiao2017fashion}, SVHN~\cite{netzer2011reading}, CIFAR10~\cite{krizhevsky2009learning} and CIFAR100~\cite{krizhevsky2009learning}.

\textbf{Kernel ridge regression}
In what follows, we deploy both Algorithm~\ref{alg:initAlg} and Algorithm~\ref{alg:smartPick} in the context of boosting the results of KRRs for RFAD~\cite{loo2022efficient}. To ensure a fair comparison, we used the same setting of~\cite{loo2022efficient}, i.e., we used the regularized Zero Component Analysis (ZCA) preprocessing. As seen from Table~\ref{tab:Kernel_based}, our algorithms boosted the performance of RFAD in terms of KRR test accuracy where in this experiment,  across the whole table, our algorithms ensured a higher accuracy, furthermore, in some cases our boosting methodology achieves almost a $5\%$  accuracy improvement (see SVHN).

\begin{table*}[htb!]
    \centering
    \caption{Distillation of five datasets. Bold numbers indicate the best performance in the context of KRR, while underlined numbers indicate the best performances in the context of neural network training. $^*$ at the end of a methods name, indicates adapting our techniques with given method.}
    \begin{adjustbox}{width=\hsize}
    \centering
    \begin{tabular}{c|cccc|||cc}
    \hline 
    & & \multicolumn{3}{c|||}{NTK-based} & \multicolumn{2}{|c}{\emph{NNLMDT}s}\\
         & Img/Cls & KIP & RFAD & $\text{RFAD}^\ast$ & DC & DC$^{\ast}$  \\
         \hline
         \multirow{ 3}{*}{MNIST} & $1$ & $95.2 \pm 0.2$ & $96.7 \pm 0.2$ & $\mathbf{97.21\pm0.2}$ & $91.7 \pm 0.5$ & $\underline{92.01 \pm 0.4}$ \\
         & $10$ & $98.4 \pm 0.0$ & $99.0 \pm 0.1$ & $\mathbf{99.14 \pm 0.0}$ & $97.4 \pm 0.2$ & $\underline{97.83 \pm 0.15}$\\
         & $50$ & $99.1 \pm 0.0$ & $99.1 \pm 0.0$ & $\mathbf{99.26\pm 0.0}$ & $98.8 \pm 0.1$ & $\underline{98.89 \pm 0.1}$\\
         \hline
         \multirow{ 3}{*}{Fashion-MNIST} & $1$ & $78.9 \pm 0.2$ & $81.6 \pm 0.6$ & $\mathbf{83.51\pm0.4}$ & $\underline{70.5 \pm 0.6}$ & $70.43 \pm 0.64$\\
         & $10$ & $87.6 \pm 0.1$ & ${90.0 \pm 0.1}$ & $\mathbf{90.64\pm0.1}$ & $\underline{82.3 \pm 0.4}$ & $82.23 \pm 0.48$ \\
         & $50$ & $90.0 \pm 0.1$ & ${91.3 \pm 0.1}$ & $\mathbf{91.7\pm0.1}$ & $83.6 \pm 0.4$ & $\underline{85.65 \pm 0.25}$ \\
         \hline
         \multirow{ 3}{*}{SVHN} & $1$ & $48.1 \pm 0.7$ & $51.4 \pm 1.3$ & $\mathbf{56.96\pm0.9}$ & $\underline{31.2 \pm 1.4}$ & $28.81 \pm 1.31$\\
         & $10$ & $75.8 \pm 0.1$ & $77.2 \pm 0.3$ &  $\mathbf{78.59\pm0.2}$ & $76.1 \pm 0.6$ & $\underline{77.14 \pm 0.72}$\\
         & $50$ & $81.3 \pm 0.2$ & $81.8 \pm 0.2$ & $\mathbf{82.48\pm0.2}$ & $\underline{82.3 \pm 0.3}$ & $81.15 \pm 0.3$\\
         \hline
         \multirow{ 3}{*}{CIFAR10} & $1$ & $59.1 \pm 0.4$ & $61.1 \pm 0.7$ & $\mathbf{64.42\pm 0.4}$ & $\underline{28.3 \pm 0.5}$ & $28.13 \pm 0.54$\\
         & $10$ & $67.0 \pm 0.4$ & $73.1 \pm 0.1$ & $\mathbf{74.29\pm0.2}$ & $44.9 \pm 0.5$ & $\underline{46.41 \pm 0.69}$\\
         & $50$ & $71.7 \pm 0.2$ & $76.1 \pm 0.3$ & $\mathbf{77.01\pm 0.2}$ & $\underline{53.9 \pm 0.5}$ & $53.4 \pm 0.5$\\
         \hline
        \multirow{2}{*}{CIFAR100} & $1$ & $31.8 \pm 0.3$ & $36.0 \pm 0.4$ & $\mathbf{38.49\pm 0.2}$ & $12.8 \pm 0.3$ & $\underline{12.81 \pm 0.34}$\\
         & $10$ & $\mathbf{46.0 \pm 0.2}$ & $44.9 \pm 0.2$ & ${45.82\pm0.1}$ & $25.2 \pm 0.3$ & $\underline{26.5 \pm 0.34}$\\
         \hline
    \end{tabular}
    \end{adjustbox}
    \label{tab:Kernel_based}
\end{table*}

\textbf{Neural network training.} We have adopted Algorithms~\ref{alg:initAlg} and~\ref{alg:smartPick} to boost the performance of a distillation technique that belongs to the \emph{NNLMDT}s family, namely (DC)~\cite{zhao2020dataset}. Similarly to before, to ensure fairness across the board, we have adopted here the same settings of~\cite{zhao2020dataset} where we modified the implementation~\cite{zhao2020dataset} to include our smart initialization (Algorithm~\ref{alg:initAlg}) and our greedy sampling (Algorithm~\ref{alg:smartPick}), and we report the average results over $5$ trials. In addition, in this experiment, we used lightweight coresets for our initialization technique, while $\phi$ function was set to be the cross-entropy loss of the network that is being trained during the distillation process of the DC method. For most of the datasets, we were able to boost the performance of the DC technique using our algorithms, while in the rest, we got comparable results.

\textbf{Boosting the results of DM.} In this experiment, we have tested the quality of our initialization technique against the $K$-center initialization technique, where $K$ denotes the number of images taken from each class. We evaluated the quality of such techniques on the DM method~\cite{zhao2023dataset}. To this end, we ran the DM method $10$ times, each with a different number of iterations from the set $\br{2000i}_{i = 1}^{10}$. Figure~\ref{fig:smartInitResults} shows that our initialization technique is comparable to the $K$-center initialization technique for the case of having $K=1$, while for the case of $K=10$, the advantage of choosing our initialization technique is more clear. As for the case of $K=50$, our method is far superior to the $K$-center initialization techniques. 

\begin{figure}[t]
    \begin{subfigure}[b]{0.33\textwidth}
        \centering
    \includegraphics[width=\textwidth]{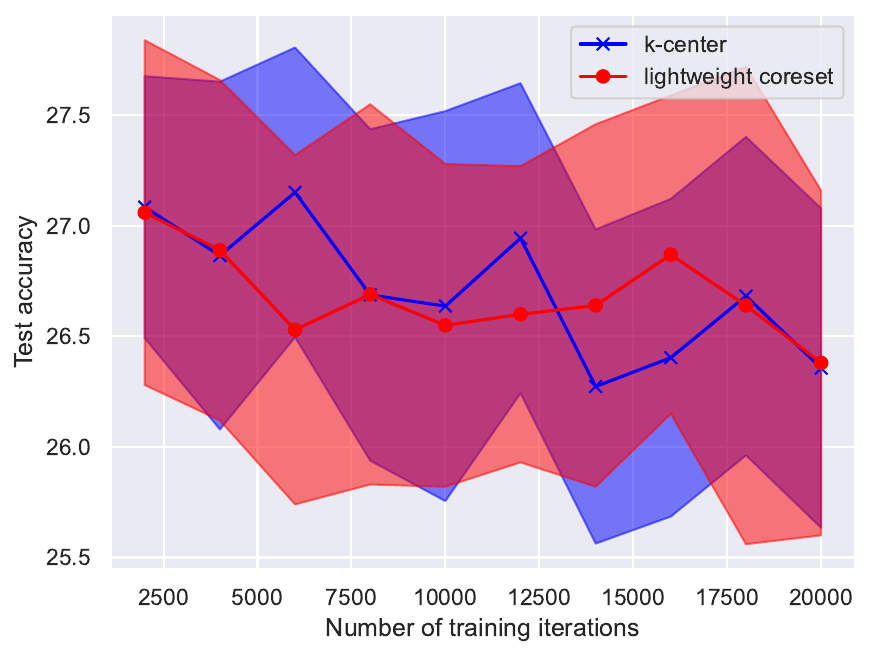}
    \caption{1 images per class}
    \label{fig:1ipcDM}
    \end{subfigure}
    \begin{subfigure}[b]{0.33\textwidth}
        \centering
    \includegraphics[width=\textwidth]{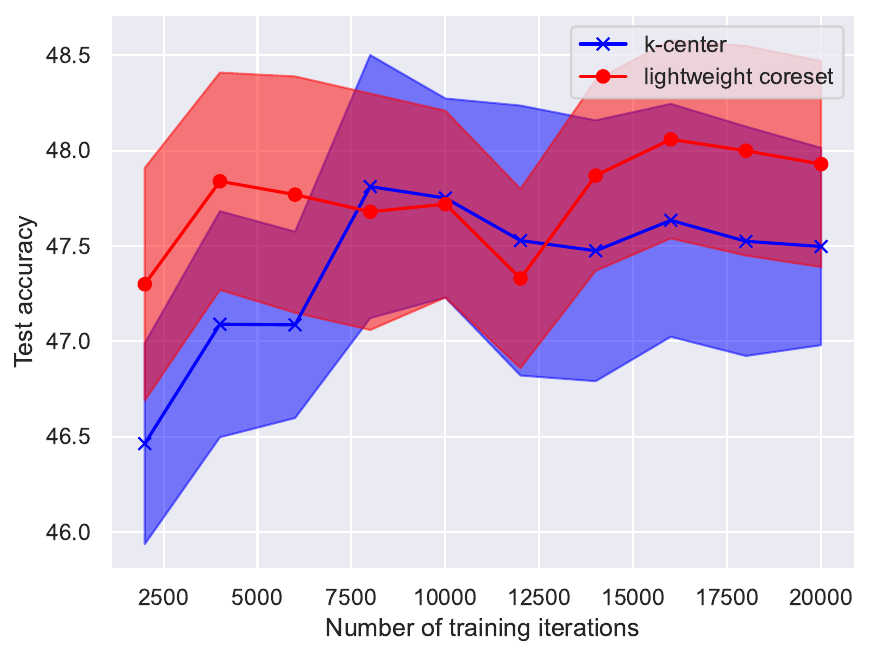}
    \caption{10 images per class}
    \label{fig:10ipcDM}
    \end{subfigure}
    \begin{subfigure}[b]{0.33\textwidth}
        \centering
    \includegraphics[width=\textwidth]{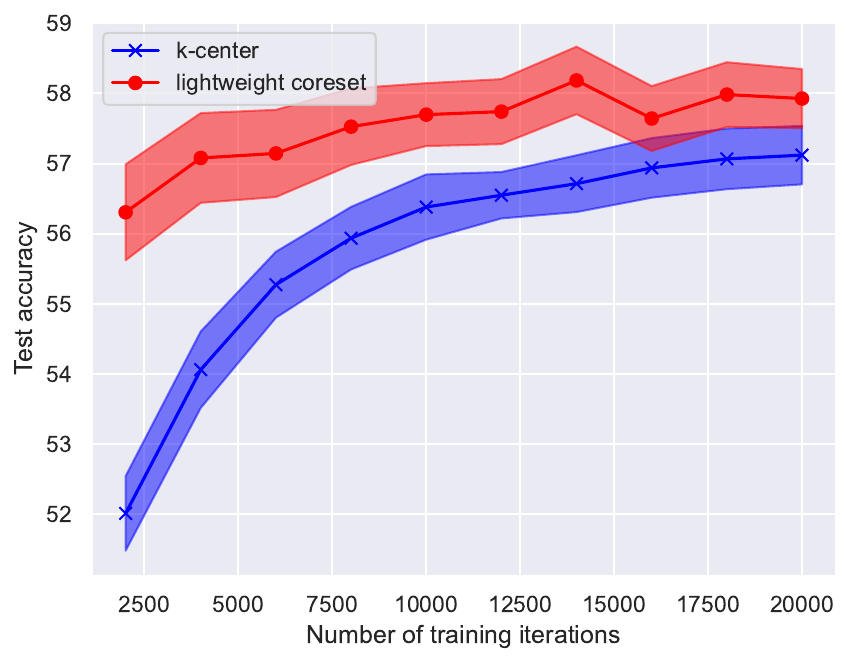}
    \caption{50 images per class}
    \label{fig:50ipcDM}
    \end{subfigure}
    \caption{Comparison between $K$-center initialization and a lightweight coreset for $K$-means}
    \label{fig:smartInitResults}
\end{figure}

\section{Conclusions and future work}
\label{sec:conclusion}
This paper provides a mechanism for enhancing the quality of various distillation techniques by leveraging key observations from data subset selection into the world of dataset distillation, both by (i) providing provable initialization techniques, and (ii) introducing smarter sampling methodology to better grasp the underlying structure of the data that the distillation method aims to achieve. 
Our experiments confirm that the choice of sampling methodology and initial set that the distillation uses is crucial as it leads to better test accuracy using sophisticated sampling methodologies.
we hope that our work will be regarded as a foundation stone in the pursuit of sophisticated sampling techniques and initialization methods for enhancing the quality of distillation methods or learning in general.

\bibliographystyle{plain}
\bibliography{main}

\appendix
\newpage
\section{Coreset tools}
\label{sec:coreset_tools}

We first define the term \emph{query space} which will aid us in simplifying the proofs as well as the corresponding theorems.

\begin{Definition}[Query space]
\label{def:querySpace}
Let $P$ be a set of $n\geq1$ points in $\REAL^d$, $w : P \to [0,\infty)$ be a non-negative weight function, and let $f: P \times X \to [0,\infty)$ denote a loss function. The tuple $(P, w, X, f)$ is called a query space.
\end{Definition}

\begin{definition}[VC-dimension~\cite{braverman2016new}]
\label{def:dimension}
For a query space $(P,w,X,f)$ and $r \in [0,\infty)$, we define 
\[
\RANGES(x,r) = \br{p \in P \mid w(p) f(p,x) \leq r},
\]
for every $x \in X$ and $r \geq 0$. The dimension of $(P, w, X, f)$ is the size $\abs{S}$ of the largest subset $S \subset P$ such that
\[
\abs{\br{S \cap \RANGES(x,r) \mid x \in X, r \geq 0 }} = 2^{\abs{S}},
\]
where $\abs{ A }$ denotes the number of points in $A$ for every $A \subseteq \REAL^d$.
\end{definition}

\begin{theorem}[Restatement of Theorem 5.5 in~\cite{braverman2016new}]
\label{thm:coresetbraverman}
%Let $P \subseteq \REAL^d$ be a set of $n$ points, and let $f : P \times \REAL^d \to [0, \infty)$ be a loss function. 
Let $(P,w)$ be a weighted set of $n$ points in $\REAL^d$, $\eps,\delta \in (0,1)$,  $X\subseteq \REAL^d$, and let $f : P \times X \to [0,\infty)$.
For every $p \in P$ define the \emph{sensitivity} of $p$ as
$
\sup_{x \in X} \frac{w(p)f(p,x)}{\sum_{q \in P} w(q)f(q,x)}
$, 
where the sup is over every $x \in X$ such that the denominator is non-zero.
Let $s: P \to [0,1]$ be a function such that $s(p)$ is an upper bound on the sensitivity of $p$.
Let $t = \sum_{p \in P} s(p)$ and $d'$ be the~\emph{VC dimension} of quadruple $(P,w,X,f)$; see Definition~\ref{def:dimension}. 
Let $c \geq 1$ be a sufficiently large constant, and let $S$ be a random sample of $\abs{S} \geq \frac{ct}{\varepsilon^2}\left(d'\log{(t)}+\log{(\frac{1}{\delta}})\right)$ 
i.i.d points from $P$, such that every $p \in P$ is sampled with probability $s(p)/t$. 
Let $v(p) = \frac{w(p)t}{s(p)\abs{S}}$ for every $p \in S$. 
Then, with probability at least $1-\delta$, $(S,v)$ is an $\varepsilon$-coreset for $(P,w,X,f)$.

\end{theorem}

\section{Random Fourier features}
\label{sec:fourier}

\begin{theorem}[Bochner's theorem \cite{bochner1933monotone}]
\label{thm:bochner}
A complex-valued function $g : \REAL^d \to \mathbb{C}$ is positive-definite if and only if it is the Fourier
Transform of a finite non-negative Borel measure $\mu$ on $\REAL^d$, i.e.,
\[
g(x) = \int\limits_{\REAL^d} \rho(\omega) e^{-ix^T\omega}d\mu(w), \quad \forall x\in \REAL^d 
\]
\end{theorem}

With the above in mind, if the Gaussian kernel is adequately scaled, Theorem~\ref{thm:bochner} guarantees that its Fourier
transforms $\rho(\omega)$ is a proper probability distribution~\cite{rahimi2007random}. This observation gave birth to the following celebrated result of~\cite{rahimi2007random}. 

We now present the main tool which is used in this paper for approximating the feature map that our Gaussian kernel aims to represent.

\begin{theorem}[Special case of Claim 1 from~\cite{rahimi2007random}]
\label{thm:rahimi}
Let $R > 0$ be a positive real number, and let $X \subseteq \REAL^d$ be a set of queries with diameter $R$. Let $k: X \times X \to \REAL$ be the Gaussian kernel, $\sigma_\rho$ be the second moment of the Fourier transform of $k$, and let $\eps \in (0,1)$. Then for $D \in \Omega\term{\frac{d}{\eps^2} \log{\term{\frac{\sigma_\rho R}{\eps}}}}$, there exists $z : M \to \REAL^D$ such that 
$\sup_{x,y \in X} \abs{z(x)^Tz(y) - k(x,y)} \leq \eps,$ with constant probability.
\end{theorem}

We now present for completeness the process needed to construct such a feature map. 

\paragraph{Overview of Algorithm~\ref{alg:fourier}.}
Given a positive-definite shift-invariant kernel function $k$, Algorithm~\ref{alg:fourier} outputs a feature map that aims to approximate the feature map that the kernel function $k$ aims to represent.
To construct such a feature map, first, Fourier transform is constructed (see Line~\ref{algFourier:line1}), followed by drawing $D$ samples i.i.d from $\REAL^d$ using the Fourier transform of $k$. We then draw $D$ i.i.d. samples from a uniform distribution $\left[ 0, 2\pi \right]$; see Lines~\ref{algFourier:line2}--\ref{algFourier:line3}.
Finally, at Line~\ref{algFourier:line4}, the feature map $z$ is constructed.

\begin{algorithm}[htb!]
\caption{$\fourier(k)$\label{alg:fourier}}
{\begin{tabbing}
\textbf{Input:} \quad\= A positive-definite shift-invariant kernel $k$.\\
\textbf{Output:} \> A randomized feature map $z : \REAL^d \to \REAL^D$ \\\>satisfying Theorem~\ref{thm:rahimi}.
\end{tabbing}}

Compute the Fourier transform $\rho$ of the kernel $k$\label{algFourier:line1}\\

Draw $D$ i.i.d. samples $\br{\omega_i \in \REAL^d}_{i=1}^D$ from $\rho$\label{algFourier:line2}\\

Draw $D$ i.i.d. samples $\br{b_i \in \REAL}_{i=1}^D$ from the uniform distribution $[0,2\pi]$\label{algFourier:line3}\\ 

\Return $z(x) := \sqrt{\frac{2}{D}} \begin{bmatrix} \cos\term{\omega_1^T x + b_1} \\ \cos\term{\omega_2^T x + b_2} \\ \vdots \\ \cos\term{\omega_D^T x + b_D}\end{bmatrix}$ \label{algFourier:line4}
\end{algorithm}

\section{Proof of our technical results}
\subsection{Proof of Theorem~\ref{thm:init}}
\label{sec:proof_of_init_theorem}

First, we present the following tool which we will use to bound our sensitivities with respect to the kernel fitting problem.

\begin{lemma}[Special case of Lemma 35,~\cite{tukan2020coresets}]
\label{lem:sens_bound}
Let $(\tilde{P},{w})$ be a weighted set of points in $\REAL^{c}$.
%and let $f : P \times \REAL^\tilde{d} \to [0,\infty)$ such that for every $p \in P$, $x \in \REAL^\tilde{d}$, %$f(p,x) = \abs{p^Tx}$.
There exists a diagonal matrix of $\Sigma \in [0, \infty)^{c \times c}$ of full rank  and an orthogonal matrix $V \in \REAL^{c \times c}$ such that every $x \in \REAL^{c}$, 
$\norm{\Sigma V^Tx}_2 \leq \sum\limits_{p \in \tilde{P}} w(p)\abs{p^Tx} \leq \sqrt{c} \norm{\Sigma V^Tx}_2.$ 
Define $U : \tilde{P} \to \REAL^{c}$ such that $U(p) = p\term{\Sigma V^T}^{-1}$ for every $p \in \tilde{P}$. Then,
\begin{enumerate*}[label=(\roman*)]
\item \label{claim:lzreg1} for every $p\in \tilde{P}$,  $\sup_{x \in \REAL^c} \frac{ w(p)\abs{p^Tx}}{\sum\limits_{p \in \tilde{P}} w(q)\abs{q^Tx}} \leq \norm{U(p)}_1$, 
\item \label{claim:lzreg2} and $\sum\limits_{p \in \tilde{P}} \sup_{x \in \REAL^c} \frac{ w(p)\abs{p^Tx}}{\sum\limits_{p \in \tilde{P}} w(q)\abs{q^Tx}} \leq c^{1.5}$.
\end{enumerate*}
The tuple $\term{U,\Sigma,V}$ is the $\ell_1$-SVD of $\tilde{P}$~\cite{tukan2020coresets}.
%i.e., 
\end{lemma}

With the above in mind, we proceed to prove Theorem~\ref{thm:init}.

\init*

\begin{proof}
\textbf{Handling NTK-based distillation techniques.} First, let $z : X \to \REAL^D$ be the result of plugging $X:=X$, $k := k$ and $\eps:=\eps$ into Theorem~\ref{thm:rahimi}. Observe that for every $x,y \in X$,
\begin{equation*}
k(x,y) - \eps \leq z(x)^T z(y) \leq k(x,y) + \eps,
\end{equation*}
where both inequalities holds by Theorem~\ref{thm:rahimi}.

Thus, it holds that for every $x \in X$
\begin{equation}
\label{eq:concentric_bounds}
\frac{1}{n}\sum\limits_{ p \in P} k(p,x) - \eps \leq \frac{1}{n}\sum\limits_{p \in P} z(p)^T z(x) \leq \frac{1}{n}\sum\limits_{p \in P} k(p,x) + \eps,
\end{equation}
where the inequalities follows since $P \subseteq X$, and $\sum\limits_{p \in P} \frac{\eps}{n} = \eps$.

Notice that for every $x,y \in X$, $z(x)^Tz(y) \geq 0$ holds only for $\eps \leq \alpha\term{R}$. With this in mind, using Lemma~\ref{lem:sens_bound} results in
\[
s(z(p)) = \sup_{x \in X} \frac{z(p)^Tz(x)}{\sum\limits_{q \in P} z(q)^T z(x)} \leq \norm{U(z(p))}_1.
\]

By plugging the bound on sensitivities, $\eps:=\eps$, and $\delta:=\delta$ into Theorem~\ref{thm:coresetbraverman}, we obtain a subset $S \subseteq P$ and a weight function $v : S \to [0, \infty)$ such that for every $x \in X$,
\begin{equation}
\label{eq:sensZP}
\begin{split}
(1 - \eps) \frac{1}{n}\sum\limits_{p \in P} z(p)^Tz(x) &\leq \sum\limits_{p \in S} v(p) z(p)^Tz(x) \\
&\leq (1 + \eps) \frac{1}{n}\sum\limits_{p \in P} z(p)^Tz(x). 
\end{split}
\end{equation}

By combining~\eqref{eq:concentric_bounds} and~\eqref{eq:sensZP}, we obtain that 
\[
(1-\eps) \frac{1}{n} k(p,x) - 2\eps \leq \sum\limits_{q \in S} v(q) z(q)^T z(x),
\]
 and 
\[
\sum\limits_{q \in S} v(q) z(q)^T z(x) \leq (1+\eps) \frac{1}{n} \sum\limits_{p \in P}k(p,x) + 2\eps
\]

The theorem for NTK-based distillation techniques then follows by combining the above two inequalities with the assumption that $\eps \in (0, \alpha\term{R}]$.

\textbf{Handling \emph{NNLMDT}s.} The coreset guarantees hold by~\cite{bachem2018scalable}; see Theorem 2 of ~\cite{bachem2018scalable}.
\end{proof}

\subsection{Proof of Theorem~\ref{thm:robin_heart}}
\robinheart*
\begin{proof}
The sensitivity of each item in $P$ with respect to some distillation loss at some iteration of the distillation procedure is heavily correlated with the accuracy of prediction. The model accuracy is directly influenced by the distilled dataset that is used to update the model at the current iteration. %This is influenced by the distilled data at that iteration. 
So in a way, the optimization involved in the definition of the sensitivity (see Theorem~\ref{thm:coresetbraverman}) is easily solvable since $X$ contains a single element related to the chosen distillation technique.

Hence, setting the sensitivity of any instance of $p$ becomes $\frac{\phi\term{p, y(p),\tilde{P},\tilde{y}}}{\sum\limits_{q \in P} \phi\term{q, y(q),\tilde{P},\tilde{y}}}$. Note that this function is also used in Line~\ref{liine:distribution} of Algorithm~\ref{alg:smartPick}.

With the above sensitivities, we can compute a weak coreset using Theorem~\ref{thm:coresetbraverman}, i.e., the coreset's provable guarantees only hold 
with respect to the distillation loss at the specific time (iteration).
Observe that since we have only one query which is the current state of the model (that is being used in the distillation technique), then the total sensitivity is $1$, where also the VC dimension is also $1$ (see Definition~\ref{def:dimension}).

Thus by Theorem~\ref{thm:coresetbraverman}, it suffices to sample only $O\term{\frac{1}{\eps^2} \term{\log{\frac{1}{\delta}}}}$, and with this, the proof of Theorem~\ref{thm:robin_heart} is concluded.

\end{proof}

\section{Extensions}
\label{sec:extensions}
\textbf{From additive to multiplicative approximation.} We show that by tuning $\varepsilon$ in Theorem~\ref{thm:rahimi}, our coreset can be associated with only a multiplicative-factor approximation.
%We show that our thery holds for
\begin{remark}
\label{remark:init}
Notice that Theorem~\ref{thm:init} holds when the approximation error used to produce the random feature map is upper bounded by the lowest value the kernel function $k$ can get over the Cartesian product of $P$ and $X$. This indicates that our coreset size is inversely proportional to the square of such a term. However, such a term is usually pessimistic as our results indicate that using such coreset with small coreset sizes, yields more informative distilled datasets.
In addition, for $\eps$ satisfying the above, our coreset gives a multiplicative approximation.
\end{remark}

\paragraph{Shift-invariant kernels -- Supporting a larger family of kernels.} Our theory holds for a large family of kernels including the Gaussian kernel, Laplacian kernel, and the Cauchy kernel as shown in the following table.

\begin{table}[htb!]
    \caption{Example of shift-invariant kernels}
    \centering
    \begin{tabular}{|c|c|}
     Kernel name &  $k(x,y)$ \\
     \hline
     Gaussian & $e^{-\norm{x-y}_2^2}$\\
     Laplacian & $e^{-\norm{x-y}_1}$ \\ 
     Cauchy & $\Pi_{i=1}^d \frac{2}{1 + \abs{e_i^T \term{x-y}}^2}$
    \end{tabular}
    \label{tab:shift_invariant}
\end{table}

\section{Limitations}
\paragraph{Bound on $\eps$ -- Theorem~\ref{thm:init}.} Theorem~\ref{thm:init} states that for certain values of $\eps$, the provable guarantees of the coreset holds. This is indeed essential for ensuring that the kernel fitting problem can be reformulated as a $\ell_1$-regression problem. Without such an assumption, it is required to split the data $P$ into two sets which are defined by the sign of $z(p)^Tz(x)$, i.e., $P_{+}\term{x} = \br{p \mid| p \in P, z(p)^Tz(x) \geq 0}$, and $P_{-}\term{x} = P \setminus P_{+}\term{x}$. The problem with such an approach is the fact that for every $x \in X$, the content of these sets would change. Consequently, without our assumption, a bound on the sensitivity becomes as hard as solving the optimization problem that the definition of the sensitivity of each point $p \in P$ entails.

\paragraph{The bound with respect to the diameter of $X$.} While such an assumption seems to limit, it is essential for the existence of a coreset for the kernel fitting problem regardless of the sensitivity bounding technique. Without such an assumption, the coreset size would be bounded from below by $\Omega{n}$ using a similar derivation to that of Theorem 3.1~\cite{tukan2023provable}.
\end{document}